\documentclass{colt2013} 

\usepackage{amsfonts,amsmath,color,float,verbatim}
\usepackage{algorithm,algorithmic}
\usepackage{multirow,array,xfrac}

\newtheorem{assumption}{Assumption}

\newcommand{\reals}{\mathbb{R}}
\newcommand{\E}{\mathbb{E}}

\newcommand{\be}{\mathbf{e}}
\newcommand{\bx}{\mathbf{x}}
\newcommand{\bw}{\mathbf{w}}
\newcommand{\bg}{\mathbf{g}}
\newcommand{\bb}{\mathbf{b}}

\newcommand{\br}{\mathbf{r}}

\newcommand{\Ocal}{\mathcal{O}}

\newcommand{\Ncal}{\mathcal{N}}

\newcommand{\Wcal}{\mathcal{W}}

\newcommand{\norm}[1]{\|#1\|}
\newcommand{\inner}[1]{\langle#1\rangle}

\newcommand{\secref}[1]{Sec.~\ref{#1}}

\newcommand{\figref}[1]{Fig.~\ref{#1}}
\renewcommand{\eqref}[1]{Eq.~(\ref{#1})}
\newcommand{\lemref}[1]{Lemma~\ref{#1}}
\newcommand{\thmref}[1]{Thm.~\ref{#1}}

\newcommand{\algref}[1]{Algorithm~\ref{#1}}

\title[Complexity of Bandit and Derivative-Free Stochastic Convex Optimization]{On the Complexity of Bandit and Derivative-Free\\ Stochastic Convex Optimization}

\coltauthor{\Name{Ohad Shamir} \Email{ohad.shamir@weizmann.ac.il}\\
\addr Microsoft Research and the Weizmann Institute of Science
}

\begin{document}

\maketitle

\begin{abstract}
The problem of stochastic convex optimization with bandit feedback (in the learning community) or without knowledge of gradients (in the optimization community) has received much attention in recent years, in the form of algorithms and performance upper bounds. However, much less is known about the inherent complexity of these problems, and there are few lower bounds in the literature, especially for nonlinear functions. In this paper, we investigate the attainable error/regret in the bandit and derivative-free settings, as a function of the dimension $d$ and the available number of queries $T$. We provide a precise characterization of the attainable performance for strongly-convex and smooth functions, which also imply a non-trivial lower bound for more general problems. Moreover, we prove that in both the bandit and derivative-free setting, the required number of queries must scale at least quadratically with the dimension. Finally, we show that on the natural class of quadratic functions, it is possible to obtain a ``fast'' $\Ocal(1/T)$ error rate in terms of $T$, under mild assumptions, even without having access to gradients. To the best of our knowledge, this is the first such rate in a derivative-free stochastic setting, and holds despite previous results which seem to imply the contrary.
\end{abstract}

\section{Introduction}

This paper considers the following fundamental question: Given an unknown convex function $F$, and the ability to query for (possibly noisy) realizations of its values at various points, how can we optimize $F$ with as few queries as possible?

This question, under different guises, has played an important role in several communities. In the optimization community, this is usually known as  ``zeroth-order'' or ``derivative-free'' convex optimization, since we only have access to function values rather than gradients or higher-order information. The goal is to return a point with small optimization error on some convex domain, using a limited number of queries. Derivative-free methods were among the earliest algorithms to numerically solve unconstrained optimization problems, and have recently enjoyed increasing interest, being especial useful in black-box situations where gradient information is hard to compute or does not exist \cite{nesterov11,StichMuGa11}. In a stochastic framework, we can only obtain noisy realizations of the function values (for instance, due to running the optimization process on sampled data). We refer to this setting as \emph{derivative-free SCO} (short for stochastic convex optimization).

In the learning community, these kinds of problems have been closely studied in the context of multi-armed bandits and (more generally) bandit online optimization, which are powerful models for sequential decision making under uncertainty \cite{CesaBianchiLu06,BubCes12}. In a stochastic framework, these settings correspond to repeatedly choosing points in some convex domain, obtaining noisy realizations of some underlying convex function's value. However, rather than minimizing optimization error, our goal is to minimize the (average) regret: roughly speaking, that the average of the function values we obtain is not much larger than the minimal function value. For example, the well-known multi-armed bandit problem corresponds to a linear function over the simplex. We refer to this setting as \emph{bandit SCO}. As will be more explicitly discussed later on, any algorithm which attains small average regret can be converted to an algorithm with the same optimization error. In other words, bandit SCO is only harder than derivative-free SCO. We note that in the context of stochastic multi-armed bandits, the potential gap between the two settings (under the terms ``cumulative regret'' and ``simple regret'') was introduced and studied in \cite{bubeck2011pure}.

When one is given gradient information, the attainable optimization error / average regret is well-known: under mild conditions, it is $\Theta(1/\sqrt{T})$ for convex functions and $\Theta(1/T)$ for strongly-convex functions, where $T$ is the number of queries \cite{Zin03,HazKa11,RakhShaSri12}. Note that these bounds do not explicitly depend on the dimension of the domain.

The inherent complexity of bandit/derivative-free SCO is not as well-understood. An important exception is multi-armed bandits, where the attainable error/regret is known to be exactly $\Theta(\sqrt{d/T})$, where $d$ is the dimension and $T$ is the number of queries\footnote{\label{footnote:mab} In a stochastic setting, a more common bound in the literature is $\Ocal(d\log(T)/T)$, but the $\Ocal$-notation hides a non-trivial dependence on the form of the underlying linear function (in multi-armed bandits terminology, a gap between the expected rewards bounded away from $0$). Such assumptions are not natural in a nonlinear bandits SCO setup, and without them, the regret is indeed $\Theta(\sqrt{d/T})$. See for instance \cite[Chapter 2]{BubCes12} for more details.} \cite{AuerCesFrSc02,AudBub09}. Linear functions over other convex domains has also been explored, with upper bounds on the order of $\Ocal(\sqrt{d/T})$ to $\Ocal(\sqrt{d^2/T})$ (e.g. \cite{abpasz11,BubCesKa12}). For linear functions over general domains, information-theoretic $\Omega(\sqrt{d^2/T})$ lower bounds have been proven in \cite{DanHaKa07,DanHaKa08,AudBubLu11}. However, these lower bounds are either on the regret (not optimization error); shown for non-convex domains; or are implicit and rely on artificial, carefully constructed domains. In contrast, we focus here on simple, natural domains and convex problems.

When dealing with more general, non-linear functions, much less is known. The problem was originally considered over 30 years ago, in the seminal work by Yudin and Nemirovsky on the complexity of optimization \cite{YudNem83}. The authors provided some algorithms and upper bounds, but as they themselves emphasize (cf. pg. 359), the attainable complexity is far from clear. Quite recently, \cite{JaNoRe12} provided an $\Omega(\sqrt{d/T})$ lower bound for strongly-convex functions, which demonstrates that the ``fast'' $\Ocal(1/T)$ rate in terms of $T$, that one enjoys with gradient information, is not possible here. In contrast, the current best-known upper bounds are $\Ocal(\sqrt[4]{d^2/T}), \Ocal(\sqrt[3]{d^2/T}), \Ocal(\sqrt{d^2/T})$ for convex, strongly-convex, and strongly-convex-and-smooth functions respectively (\cite{FlaxKaMc05,AgDeXi10}); And a $\Ocal(\sqrt{d^{32}/T})$ bound for convex functions (\cite{AgFoHsKaRa11}), which is better in terms of dependence on $T$ but very bad in terms of the dimension $d$.

In this paper, we investigate the complexity of bandit and derivative-free stochastic convex optimization, focusing on nonlinear functions, with the following contributions (see also the summary in Table \ref{table:results}):
\begin{itemize}
    \item We prove that for strongly-convex and smooth functions, the attainable error/regret is exactly $\Theta(\sqrt{d^2/T})$. This has three important ramifications: First of all, it settles the question of attainable performance for such functions, and is the first sharp characterization of complexity for a general nonlinear bandit/derivative-free class of problems. Second, it proves that the required number of queries $T$ in such problems must scale quadratically with the dimension, even in the easier optimization setting, and in contrast to the linear case which often allows linear scaling with the dimension. Third, it formally provides a natural $\Omega(\sqrt{d^2/T})$ lower bound for more general classes of convex problems.
    \item We analyze an important special case of strongly-convex and smooth functions, namely quadratic functions. We show that for such functions, one can (efficiently) attain $\Theta(d^2/T)$ optimization error, and that this rate is sharp. To the best of our knowledge, it is the first general class of nonlinear functions for which one can show a ``fast rate'' (in terms of $T$) in a derivative-free stochastic setting. In fact, this may seem to contradict the result in \cite{JaNoRe12}, which shows an $\Omega(\sqrt{d/T})$ lower bound on quadratic functions. However, as we explain in more detail later on, there is no contradiction, since the example establishing the lower bound of \cite{JaNoRe12} imposes an extremely small domain (which actually decays with $T$), while our result holds for a fixed domain. Although this result is tight, we also show that under more restrictive assumptions on the noise process, it is sometimes possible to obtain better error bounds, as good as $\Ocal(d/T)$.
    \item We prove that even for quadratic functions, the attainable average \emph{regret} is exactly $\Theta(\sqrt{d^2/T})$, in contrast to the $\Theta(d^2/T)$ result for optimization error. This shows there is a real gap between what can be obtained for derivative-free SCO and bandit SCO, without any specific distributional assumptions. Again, this stands in contrast to settings such as multi-armed bandits, where there is no difference in their distribution-free performance.
\end{itemize}

We emphasize that our upper bounds are based on the assumption that the function minimizer is bounded away from the domain boundary, or that we can query points slightly outside the domain. However, we argue that this assumption is not very restrictive in the context of strongly-convex functions (especially in learning applications), where the domain is often $\reals^d$, and a minimizer always exists.

The paper is structured as follows: In \secref{sec:preliminaries}, we formally define the setup and introduce the notation we shall use in the remainder of the paper. For clarity of exposition, we begin with the case of quadratic functions in \secref{sec:quadratic}, providing algorithms, upper and lower bounds. The tools and insights we develop for the quadratic case will allow us to tackle the more general strongly-convex-and-smooth setting in \secref{sec:strongly_convex}. We end the main part of the paper with a summary and discussion of open problems in \secref{sec:discussion}. In Appendix \ref{sec:ridge}, we demonstrate that one can obtain improved performance in the quadratic case, if we're considering more specific natural noise processes. Additional proofs are presented in Appendix \ref{sec:technical}.


\setlength{\fboxsep}{1pt}

\begin{table}
\renewcommand{\arraystretch}{1.5}
\begin{tabular}{|c||c|c||c|c|}
\hline
& \multicolumn{2}{c||}{\textbf{Optimization Error}} & \multicolumn{2}{c|}{\textbf{Average Regret}}\\
\hline
\textbf{Function Type} & \textbf{$\Ocal(\cdot)$} & \textbf{$\Omega(\cdot)$} & \textbf{$\Ocal(\cdot)$} & \textbf{$\Omega(\cdot)$} \\
\hline\hline
Quadratic & \multicolumn{2}{c||}{\framebox{$\mathbf{\frac{d^2}{T}}$}} & \multicolumn{2}{c|}{\framebox{$\mathbf{\sqrt{\frac{d^2}{T}}}$}}\\
\hline
Str. Convex and Smooth & \multicolumn{4}{c|}{\framebox{$\mathbf{\sqrt{\frac{d^2}{T}}}$}}\\
\hline
Str. Convex & $\min\left\{\sqrt[3]{\frac{d^2}{T}},\sqrt{\frac{d^{32}}{T}}\right\}$ & \framebox{$\mathbf{\sqrt{\frac{d^2}{T}}}$} & $\min\left\{\sqrt[3]{\frac{d^2}{T}},\sqrt{\frac{d^{32}}{T}}\right\}$ & \framebox{$\mathbf{\sqrt{\frac{d^2}{T}}}$}\\
\hline
Convex & $\min\left\{\sqrt[4]{\frac{d^2}{T}},\sqrt{\frac{d^{32}}{T}}\right\}$ & \framebox{$\mathbf{\sqrt{\frac{d^2}{T}}}$} & $\min\left\{\sqrt[4]{\frac{d^2}{T}},\sqrt{\frac{d^{32}}{T}}\right\}$ & \framebox{$\mathbf{\sqrt{\frac{d^2}{T}}}$}\\
\hline
\end{tabular}
\renewcommand{\arraystretch}{1}
\caption{A summary of the complexity upper bounds ($\Ocal(\cdot)$) and lower bounds ($\Omega(\cdot)$), for derivative-free stochastic convex optimization (optimization error) and bandit stochastic convex optimization (average regret), for various function classes, in terms of the dimension $d$ and the number of queries $T$. The boxed results are shown in this paper. The upper bounds for the convex and strongly convex case combine results from \cite{FlaxKaMc05,AgDeXi10,AgFoHsKaRa11}. The table shows dependence on $d,T$ only and ignores other factors and constants.}
\label{table:results}
\end{table}

\section{Preliminaries}\label{sec:preliminaries}

Let $\norm{\cdot}$ denote the standard Euclidean norm. We let $F(\cdot):\Wcal\mapsto \reals$ denote the convex function of interest, where $\Wcal\subseteq \reals^d$ is a (closed) convex domain. We say that $F$ is $\lambda$-strongly convex, for $\lambda>0$, if for any $\bw,\bw'\in \Wcal$ and any subgradient $\bg$ of $F$ at $\bw$, it holds that $F(\bw')\geq F(\bw)+\inner{\bg,\bw'-\bw}+\frac{\lambda}{2}\norm{\bw'-\bw}^2$. Intuitively, this means that we can \emph{lower bound} $F$ everywhere by a quadratic function of fixed curvature. We say that $F$ is $\mu$-smooth if for any $\bw,\bw'\in \Wcal$, and any subgradient $\bg$ of $F$ at $\bw$, it holds that $F(\bw')\leq F(\bw)+\inner{\bg,\bw'-\bw}+\frac{\mu}{2}\norm{\bw'-\bw}^2$. Intuitively, this means that we can \emph{upper-bound} $F$ everywhere by a quadratic function of fixed curvature. We let $\bw^*\in \Wcal$ denote a minimizer of $F$ on $\bw$. To prevent trivialities, we consider in this paper only functions whose optimum $\bw^*$ is known beforehand to lie in some bounded domain (even if $\Wcal$ is large or all of $\reals^d$), and the function is Lipschitz in that domain.

The learning/optimization process proceeds in $T$ rounds. Each round $t$, we pick and query a point $\bw_t\in \Wcal$, obtaining an independent realization of $F(\bw)+\xi_{\bw}$, where $\xi_{\bw}$ is an unknown zero-mean random variable, such\footnote{We note that this slightly deviates from the more common assumption in the bandits/derivative-free SCO setting that $\E[\xi_{\bw}^2]\leq \Ocal(1)$. While such assumptions are equivalent for bounded $\Wcal$, we also wish to consider cases with unrestricted domains $\Wcal=\reals^d$. In that case, assuming $\E[\xi_{\bw}^2]\leq \Ocal(1)$ may lead to trivialities in the derivative-free setting. For example, consider the case where $F(\bw)=\bw^\top A \bw+\bb^\top \bw$. Then for any $\bw$ and any $\xi_{\bw}$ with uniformly bounded variance, we can get a virtually noiseless estimate of $\bw^\top A \bw$ by picking $\bw'=c\bw$ for some large $c$ and computing $\frac{1}{c^2}\left(F(\bw')+\xi_{\bw'}\right)$. Variants of this idea will also allow virtually noiseless estimates of the linear term.} that $\E[\xi_{\bw}^2]\leq \max\left\{1,\norm{\bw}^2\right\}$. In the bandit SCO setting, our goal is to minimize the \emph{expected average regret}, namely
\[
\E\left[\frac{1}{T}\sum_{t=1}^{T}F(\bw_t)-F(\bw^*)\right],
\]
whereas in the derivative-free SCO setting, our goal is to compute, based on $\bw_1,\ldots,\bw_T$ and the observed values, some point $\bar{\bw}_T\in\Wcal$, such that the \emph{expected optimization error}
\[
\E\left[F(\bar{\bw}_T)-F(\bw^*)\right],
\]
is as small as possible. We note that given a bandit SCO algorithm with some regret bound, one can get a derivative-free SCO algorithm with the same optimization error bound: we simply run the stochastic bandit algorithm, getting $\bw_1,\ldots,\bw_T$, and returning $\frac{1}{T}\sum_{t=1}^{T}\bw_t$. By Jensen's inequality, the expected optimization error is at most the expected average regret with respect to $\bw_1,\ldots,\bw_T$. Thus, bandit SCO is only harder than derivative-free SCO.

In this paper, we provide upper and lower bounds on the attainable optimization error / average regret, as a function of the dimension $d$ and the number of rounds/queries $T$. For simplicity, we focus here on bounds which hold in expectation, and an interesting point for further research is to extend these to bounds on the actual error/regret, which hold with high probability.

\section{Quadratic Functions}\label{sec:quadratic}

In this section, we consider the class of quadratic functions, which have the form
\[
F(\bw) = \bw^\top A \bw + \bb^{\top}\bw+c
\]
where $A$ is positive-definite (with a minimal eigenvalue bounded away from $0$). Moreover, to make the problem well-behaved, we assume that $A$ has a spectral norm of at most $1$, and that $\norm{\bb}\leq 1, |c|\leq 1$. We note that if the norms are bounded but larger than $1$, this can be easily handled by rescaling the function. It is easily seen that such functions are both strongly convex and smooth. Moreover, this is a natural and important class of functions, which in learning applications appears, for instance, in the context of least squares and ridge regression. Besides providing new insights for this class, we will use the techniques developed here later on, in the more general case of strongly-convex and smooth functions.

\subsection{Upper Bounds}

We begin by showing that for derivative-free SCO, one can obtain an optimization error bound of $\Ocal(d^2/T)$. To the best of our knowledge, this is the first example of a derivative-free stochastic bound scaling as $\Ocal(1/T)$ for a general class of nonlinear functions, as opposed to $\Ocal(1/\sqrt{T})$. However, to achieve this result, we need to make the following mild assumption:
\begin{assumption}\label{assump}
At least one of the following holds for some fixed $\epsilon\in (0,1]$:
 \begin{itemize}
    \item The quadratic function attains its minimum $\bw^*$ in the domain $\Wcal$, and the Euclidean distance of $\bw^*$ from the domain boundary is at least $\epsilon$.
    \item We can query not just points in $\Wcal$, but any point whose distance from $\Wcal$ is at most $\epsilon$.
  \end{itemize}
\end{assumption}
With strongly-convex functions, the most common case is that $\Wcal=\reals^d$, and then both cases actually hold for any value of $\epsilon$. Even in other situations, one of these assumptions virtually always holds. Note that we crucially rely here on the strong-convexity assumption: with (say) linear functions, the domain must always be bounded and the optimum always lies at the boundary of the domain.

With this assumption, the bound we obtain is on the order of $d^2/\epsilon^2 T$. As discussed earlier, \cite{JaNoRe12} recently proved a $\Omega(\sqrt{d/T})$ lower bound for derivative-free SCO, which actually applies to quadratic functions. This does not contradict our result, since in their example the diameter of $\Wcal$ (and hence also $\epsilon$) decays with $T$. In contrast, our $\Ocal(d^2/T)$ bound holds for fixed $\epsilon$, which we believe is natural in most applications.

To obtain this behavior, we utilize a well-known $1$-point gradient estimate technique, which allows us to get an unbiased estimate of the gradient at any point by randomly querying for a (noisy) value of the function around it (see \cite{YudNem83,FlaxKaMc05}). Our key insight is that whereas for general functions one must query very close to the point of interest (scaling to $0$ with $T$), quadratic functions have additional structure which allows us to query relatively far away, allowing gradient estimates with much smaller variance.

The algorithm we use is presented as \algref{alg:quadratic}, and is computationally efficient. It uses a modification $\bar{\Wcal}$ of the domain $\Wcal$, defined as follows. First, we let $B$ denote some known upper bound on $\norm{\bw^*}$. If the first alternative of assumption \ref{assump} holds, then $\bar{\Wcal}$ consists of all points in $\Wcal\cap \{\bw:\norm{\bw}\leq B\}$, whose distance from $\Wcal$'s boundary is at least $\epsilon$. If the second alternative holds, then $\bar{\Wcal}=\Wcal\cap \{\bw:\norm{\bw}\leq B\}$. Note that under any alternative, it holds that $\bar{\Wcal}$ is convex, that $\norm{\bw_t}\leq B$, that $\bw^*\in \bar{\Wcal}$, and that our algorithm always queries at legitimate points. In the pseudocode, we use $\Pi_{\bar{\Wcal}}$ to denote projection on $\bar{\Wcal}$. For simplicity, we assume that $T/2$ is an integer and that $\bar{\Wcal}$ includes the origin $\mathbf{0}$.

\begin{algorithm}
\caption{Derivative-Free SCO Algorithm for Strongly-Convex Quadratic Functions}
\label{alg:quadratic}
\begin{algorithmic}
\STATE Input: Strong convexity parameter $\lambda>0$; Distance parameter $\epsilon\in (0,1]$
\STATE Initialize $\bw_1=\mathbf{0}$.
\FOR{$t=1,\ldots,T-1$}
    \STATE Pick $\br\in \{-1,+1\}^d$ uniformly at random
    \STATE Query noisy function value $v$ at point $\bw_t+\frac{\epsilon}{\sqrt{d}}\br$
    \STATE Let $\tilde{\bg}=\frac{\sqrt{d}v}{\epsilon}\br$
    \STATE Let $\bw_{t+1} = \Pi_{\bar{\Wcal}}\left(\bw_{t}-\frac{1}{\lambda t}\tilde{\bg}\right)$
\ENDFOR
\STATE Return $\bar{\bw}_T=\frac{2}{T}\sum_{t=T/2}^{T}\bw_{t}$.
\end{algorithmic}
\end{algorithm}

The following theorem quantifies the optimization error of our algorithm.
\begin{theorem}\label{thm:quadup}
Let $F(\bw)= \bw^\top A \bw + \bb^{\top}\bw+c$ be a $\lambda$-strongly convex function, where $\norm{A}_2,\norm{\bb},|c|$ are all at most $1$, and suppose the optimum $\bw^*$ has a norm of at most $B$. Then under Assumption \ref{assump}, the point $\bar{\bw}_T$ returned by \algref{alg:quadratic} satisfies
\[
\E\left[F(\bar{\bw}_T)- F(\bw^*)\right] \leq \frac{4(4+5\log(2))(B+1)^4}{\lambda \epsilon^2} \frac{d^2}{T}.
\]
\end{theorem}
Note that returning $\bar{\bw}_T$ as the average over the last $T/2$ iterates (as opposed to averaging over all iterates) is necessary to avoid $\log(T)$ factors \cite{RakhShaSri12}.

As an interesting side-note, we conjecture that a gradient-based approach is crucial here to obtain $\Ocal(1/T)$ rates (in terms of $T$). For example, a different family of derivative-free methods (see for instance \cite{YudNem83,AgFoHsKaRa11,JaNoRe12}) is based on a type of noisy binary search, where a few strategically selected points are repeatedly sampled in order to estimate which of them has a larger/smaller function value. This is used to shrink the feasible region where the optimum $\bw^*$ might lie. Since it is generally impossible to estimate the mean of noisy function values at a rate better than $\Ocal(1/\sqrt{T})$, it is not clear if one can get an optimization rate faster than $\Ocal(1/\sqrt{T})$ with such methods.

The proof of the theorem relies on the following key lemma, whose proof appears in the appendix.
\begin{lemma}\label{lem:momentbounds}
For any $\bw_t$, we have that
\[
\E_{\br,v}[\tilde{\bg}] = \nabla F(\bw_t)
\]
and
\[
\E_{\br,v}[\norm{\tilde{\bg}}^2]\leq \frac{4d^2(B+1)^4}{\epsilon^2}.
\]
\end{lemma}

This lemma implies that \algref{alg:quadratic} essentially performs stochastic gradient descent over the strongly-convex function $F(\bw)$, where the gradient estimates are unbiased and with bounded second moments. The returned point is a suffix-average of the last $T/2$ iterates. Using a convergence analysis for stochastic gradient descent with suffix-averaging \cite[Theorem 5]{RakhShaSri12}, and plugging in the bounds of \lemref{lem:momentbounds}, we get \thmref{thm:quadup}.

\subsection{Lower Bounds}

In this subsection, we prove that the upper bound obtained in \thmref{thm:quadup} is essentially tight: namely, up to constants, the worst-case error rate one can obtain for derivative-free SCO of quadratic functions is order of $d^2/T$. Besides showing that the algorithm above is essentially optimal, it implies that even for extremely nice strongly-convex functions and domains, the number of queries required to reach some fixed accuracy scales \emph{quadratically} with the dimension $d$. This stands in contrast to the case of linear functions, where the provable query complexity often scales linearly with $d$.

\begin{theorem}\label{thm:quadlow}
Let the number of rounds $T$ be fixed. Then for any (possibly randomized) querying strategy, there exists a quadratic function of the form $F(\bw)=\frac{1}{2}\norm{\bw}^2-\inner{\be,\bw}$, which is minimized at $\be$ where $\norm{\be}\leq 1$,  such that the resulting $\bar{\bw}_T$ satisfies
\[
\E[F(\bar{\bw}_T)-F(\bw^*)] \geq 0.01 \min \left\{1,\frac{d^2}{T}\right\}.
\]
\end{theorem}
Note that since $\norm{\be}\leq 1$, we know in advance that the optimum must lie in the unit Euclidean ball. Despite this, the lower bound holds even if we do not restrict at all the domain in which we are allowed to query - i.e., it can even be all of $\reals^d$.

\begin{proof}
The proof technique is inspired by a lower bound which appears in \cite{arcandav11}, in the different context of compressed sensing. The argument also bears some close similarities to the proof of Assouad's lemma (see \cite{cybakov2009introduction}).

We will exhibit a distribution over quadratic functions $F$, such that in expectation over this distribution, any querying strategy will attain $\Omega(d^2/T)$ optimization error. This implies that for any querying strategy, there exists some deterministic $F$ for which it will have this amount of error.

The functions we shall consider are
\[
F_\be(\bw) = \frac{1}{2}\norm{\bw}^2-\inner{\be,\bw},
\]
where $\be$ is drawn uniformly from $\left\{-\mu,\mu\right\}^d$, with $\mu\in (0,1/\sqrt{d})$ being a parameter to be specified later. Moreover, we will assume that the noise $\xi_{\bw}$ is a Gaussian random variable with zero mean and standard deviation $\max\left\{1,\norm{\bw}^2\right\}$.

By definition of $1$-strong convexity, it is easy to verify that $F_{\be}(\bw)-F_{\be}(\be)\geq \frac{1}{2}\norm{\bw-\be}^2$. Thus, the expected optimization error (over the querying strategy) is at least
\begin{equation}\label{eq:subopt1}
\E[F_{\be}(\bar{\bw}_T)-F_{\be}(\be)] \geq \E\left[\frac{1}{2}\norm{\bar{\bw}_T-\be}^2\right] \geq \E\left[\frac{1}{2}\sum_{i=1}^{d}(\bar{w}_{i}-e_i)^2\right] \geq \E\left[\frac{\mu^2}{2}\sum_{i=1}^{d}\mathbf{1}_{\bar{w_i}e_i < 0}\right].
\end{equation}
We will assume that the querying strategy is deterministic: $\bw_t$ is a deterministic function of the previous query values $v_1,v_2,\ldots,v_{t-1}$ at $\bw_1,\ldots,\bw_{t-1}$. This assumption is without loss of generality, since any random querying strategy can be seen as a randomization over deterministic querying strategy. Thus, a lower bound which holds uniformly for any deterministic querying strategy would also hold over a randomization.

To lower bound \eqref{eq:subopt1}, we use the following key lemma, which relates this to the question of how informative are the query values (as measured by Kullback-Leibler or KL divergence) for determining the sign of $\be$'s coordinates. Intuitively, the more similar the query values are, the smaller is the KL divergence and the harder it is to distinguish the true sign of each $e_i$, leading to a larger lower bound. The proof appears in the appendix.

\begin{lemma}\label{lem:probdkl}
Let $\be$ be a random vector, none of whose coordinates is supported on $0$, and let $v_1,v_2,\ldots,v_T$ be a sequence of query values obtained by a deterministic strategy returning a point $\bar{\bw}_T$ (so that the query location $\bw_t$ is a deterministic function of $v_1,\ldots,v_{t-1}$, and $\bar{\bw}_T$ is a deterministic function of $v_1,\ldots,v_T$). Then we have
\[
\E\left[\sum_{i=1}^{d}\mathbf{1}_{\bar{w}_i e_i < 0}\right]
~\geq~
\frac{d}{2}\left(1-\sqrt{\frac{1}{d}\sum_{i=1}^{d}\sum_{t=1}^{T}U_{t,i}}\right),
\]
where
\[
U_{t,i} = \sup_{\{e_j\}_{j\neq i}}
D_{kl}\left(\Pr\left(v_t|e_i>0,\{e_j\}_{j\neq i},\{v_l\}_{l=1}^{{t-1}}\right) ~||~
\Pr\left(v_t|e_i<0,\{e_j\}_{j\neq i},\{v_l\}_{l=1}^{{t-1}}\right)
\right)
\]
and $D_{kl}$ represents the KL divergence between two distributions.
\end{lemma}

Using \lemref{lem:probdkl}, we can get a lower bound for the above, provided an upper bound on the $U_{t,i}$'s. To analyze this, consider any fixed values of $\{e_j\}_{j\neq i}$, and any fixed values of $v_1,\ldots,v_{t-1}$. Since the querying strategy is assumed to be deterministic, it follows that $\bw_t$ is uniquely determined. Given this $\bw_t$, the function value $v_t$ equals
\begin{equation}\label{eq:P}
F_{\be}(\bw_t) = \left(\frac{1}{2}\norm{\bw_t}^2+\sum_{j\neq i}e_j w_{t,j}\right)+\mu w_{t,i}+\xi_{\bw_t}
\end{equation}
conditioned on $e_i>0$, and
\begin{equation}\label{eq:Q}
F_{\be}(\bw_t) = \left(\frac{1}{2}\norm{\bw_t}^2+\sum_{j\neq i}e_j w_{t,j}\right)-\mu w_{t,i}+\xi_{\bw_t}
\end{equation}
conditioned on $e_i<0$.
Comparing \eqref{eq:P} and \eqref{eq:Q}, we notice that they both represent a Gaussian distribution (due to the $\xi_{\bw_t}$ noise term), with standard deviation $\max\left\{1,\norm{\bw_t}^2\right\}$ and means seperated by $2\mu w_{t,i}$. To bound the divergence, we use the following standard result on the KL divergence between two Gaussians \cite{kullback59}:

\begin{lemma}\label{lem:gaussians}
Let $\Ncal(\mu,\sigma^2)$ represent a Gaussian distribution variable with mean $\mu$ and variance $\sigma^2$. Then
\[
D_{kl}\left(\Ncal(\mu_1,\sigma^2)||\Ncal(\mu_2,\sigma^2)\right)
~=~ \frac{(\mu_1-\mu_2)^2}{2\sigma^2}
\]
\end{lemma}

Using this lemma, it follows that
\begin{align*}
& D_{kl}\left(P(v_t|v_{1},\ldots,v_{t-1})||Q(v_t|v_1,\ldots,v_{t-1})\right)
~\leq~
\frac{(2\mu w_{t,i})^2}{2\max\left\{1,\norm{\bw_t}^4\right\}}
~=~
\frac{2\mu^2 w_{t,i}^2}{\max\left\{1,\norm{\bw_t}^4\right\}}.
\end{align*}
Plugging this upper bound on the $U_{t,i}$'s in \lemref{lem:probdkl}, we can further lower bound on the expected optimization error from  \eqref{eq:subopt1} by
\begin{align}
&\frac{d\mu^2}{4}\left(1-\sqrt{\frac{1}{d}\sum_{t=1}^{T}\sum_{i=1}^{d}
\frac{2\mu^2 w_{t,i}^2}{\max\left\{1,\norm{\bw_t}^4\right\}}}\right)
~=~
\frac{d\mu^2}{4}\left(1-\sqrt{\frac{2\mu^2}{d}\sum_{t=1}^{T}
\frac{ \norm{\bw}_t^2}{\max\left\{1,\norm{\bw_t}^4\right\}}}\right)\notag\\
&=~
\frac{d\mu^2}{4}\left(1-\sqrt{\frac{2\mu^2}{d}\sum_{t=1}^{T}
\min\left\{\norm{\bw_t}^2,\frac{1}{\norm{\bw_t}^2}\right\}}\right)
~\geq~
\frac{d\mu^2}{4}\left(1-\sqrt{\frac{2T\mu^2}{d}}\right)\label{eq:quadloweq}.
\end{align}
Finally, we choose $\mu=\min\{1/\sqrt{d},\sqrt{d/4T}\}$, and obtain a lower bound of
\[
\frac{1}{4}\left(1-\frac{1}{\sqrt{2}}\right)\min\left\{1,\frac{d^2}{4T}\right\}
 > 0.01 \min\left\{1,\frac{d^2}{T}\right\}
\]
as required.
\end{proof}

The theorem above applies to the optimization error for derivative-free SCO. We now turn to deal with the case of bandit SCO and regret, showing an $\Omega(\sqrt{d^2/T})$ lower bound. Since the derivative-free SCO bound was $\Theta(d^2/T)$, the result implies a real gap between what can be obtained in terms of average regret, as opposed to optimization error, without any specific distributional assumptions. This stands in contrast to settings such as multi-armed bandits, where the construction implying the known $\Omega(\sqrt{d/T})$ lower bound (e.g. \cite{CesaBianchiLu06}) applies equally well to derivative-free and bandit SCO (see \cite{bubeck2011pure}).

\begin{theorem}\label{thm:quadlowregret}
Let the number of rounds $T$ be fixed. Then for any (possibly randomized) querying strategy, there exists a quadratic function of the form $F(\bw)=\frac{1}{2}\norm{\bw}^2-\inner{\be,\bw}$, which is minimized at $\be$ where $\norm{\be}\leq 1/2$, such that
\[
\E\left[\frac{1}{T}\sum_{t=1}^{T} F(\bw_t)-F(\bw^*)\right] \geq 0.02\min\left\{1,\sqrt{\frac{d^2}{T}}\right\}.
\]
\end{theorem}
Note that our lower bound holds even when the domain is unrestricted (the algorithm can pick any point in $\reals^d$). Moreover, the lower bound coincides (up to a constant) with the $\Ocal(\sqrt{d^2/T})$ regret upper-bound shown for strongly-convex and smooth functions in \cite{AgDeXi10}. This shows that for strongly-convex and smooth functions, the minimax average regret is $\Theta(\sqrt{d^2/T})$. Also, the lower bound implies that one cannot hope to obtain average regret better than $\sqrt{d^2/T}$ for more general bandit problems, such as strongly-convex or even convex problems.

The proof relies on techniques similar to the lower bound of \thmref{thm:quadlow}, with a key additional insight. Specifically, in \thmref{thm:quadlow}, the lower bound obtained actually depends on the norm of the points $\bw_1,\ldots,\bw_T$ (see \eqref{eq:quadloweq}), and the optimal $\bw^*$ has a very small norm. In a regret minimization setting the points $\bw_1,\ldots,\bw_T$ cannot be too far from $\bw^*$, and thus must have a small norm as well, leading to a stronger lower bound than that of \thmref{thm:quadlow}. The formal proof appears in the appendix.

\section{Strongly Convex and Smooth Functions}\label{sec:strongly_convex}

We now turn to the more general case of strongly convex and smooth functions. First, we note that in the case of functions which are both strongly convex and smooth, \cite[Theorem 14]{AgDeXi10} already provided an $\Ocal(\sqrt{d^2/T})$ average regret bound (which holds even in a non-stochastic setting). The main result of this section is a \emph{matching} lower bound, which holds even if we look at the much easier case of derivative-free SCO. This lower bound implies that the attainable error for strongly-convex and smooth functions is order of $\sqrt{d^2/T}$, and at least $\sqrt{d^2/T}$ for any harder setting.

\begin{theorem}\label{thm:convex}
Let the number of rounds $T$ be fixed. Then for any (possibly randomized) querying strategy, there exists a function $F$ over $\reals^d$ which is $0.5$-strongly convex and $3.5$-smooth; Is $4$-Lipschitz over the unit Euclidean ball; has a global minimum in the unit ball; And such that the resulting $\bar{\bw}_T$ satisfies
\[
\E[F(\bar{\bw}_T)-F(\bw^*)] \geq 0.004 \min\left\{1,\sqrt{\frac{d^2}{T}}\right\}.
\]
\end{theorem}
Note that we made no attempt to optimize the constant.

The general proof technique is rather similar to that of \thmref{thm:quadlow}, but the construction is a bit more intricate. Specifically, letting $\mu>0$ be a parameter to be determined later, we look at functions of the form
\[
F_{\be}(\bw) = \norm{\bw}^2-\sum_{i=1}^{d}\frac{e_i w_i}{1+(w_i/e_i)^2},
\]
where $\be$ is uniformly distributed on $\left\{-\mu,+\mu\right\}^d$. To see the intuition behind this choice, let us consider the one-dimensional case ($d=1$). Recall that in the quadratic setting, the function we considered (in one dimension) was of the form
\[
F_{e}(w) = \frac{1}{2}w^2-e w,
\]
where $e$ was chosen uniformly at random from $\{-\mu,+\mu\}$, and $\mu$ is a ``small'' number. Thus, the optimum is at either $-\mu$ or $\mu$, and the difference $|F_{\mu}(w)-F_{-\mu}(w)|$ at these optima is order of $\mu^2$. However, by picking $w=\Theta(1)$, the difference $|F_{\mu}(w)-F_{-\mu}(w)|$ is on the order of $\mu$ - much larger than the difference close to the optimum, which is order of $\mu^2$. Therefore, by querying for $w$ far from the optimum, and getting noisy values of $F_{e}$, it is easier to distinguish whether we are dealing with $e=+\mu$ or $e=-\mu$, leading to a $d^2/T$ optimization error bound. In contrast, the function we consider here (in the one-dimensional case) is of the form
\begin{equation}\label{eq:1dfunc}
F_{e}(w)= w^2-\frac{e w}{1+(w/e)^2}.
\end{equation}
This form is carefully designed so that $|F_{\mu}(w)-F_{-\mu}(w)|$ is order of $\mu^2$, not just at the optima of $F_{\mu}$ and $F_{-\mu}$, but for \emph{all} $w$. This is because of the additional denominator, which makes the function closer and closer to $w^2$ the larger $w$ is - see \figref{fig:functions} for a graphical illustration. As a result, no matter how the function is queried, distinguishing the choice of $\mu$ is difficult, leading to the strong lower bound of \thmref{thm:convex}. A formal proof is presented in the appendix.

\begin{figure}[t]
\begin{center}
\includegraphics[scale=0.7]{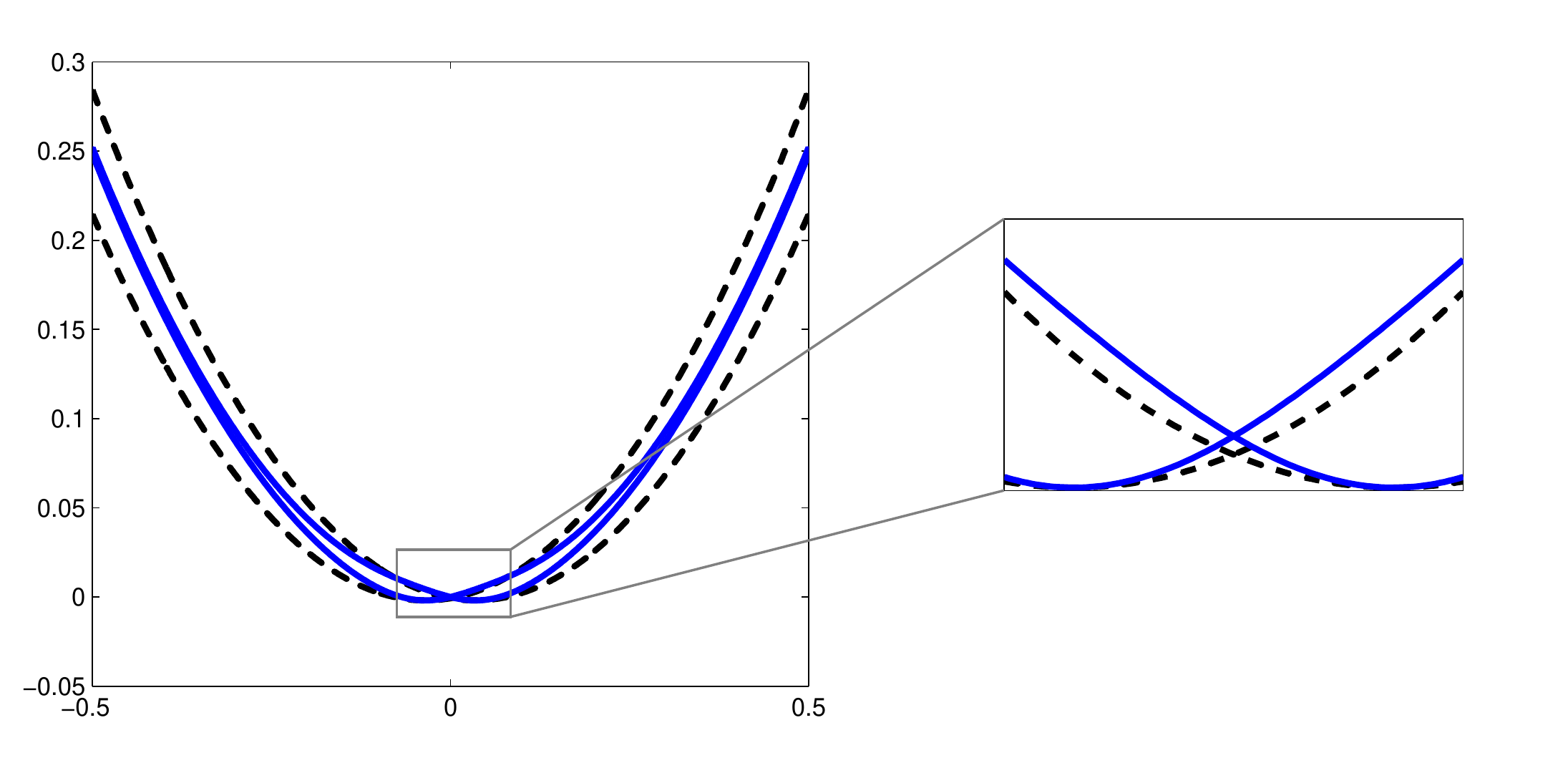}
\end{center}
\vskip -1.0cm
\caption{The two solid blue lines represents $F_e(w)$ as in \eqref{eq:1dfunc}, for $e=0.1$ and $e=-0.1$, whereas the two dashed black lines represent two quadratic functions with similar minimum points. Close to the minima, $F_e(w)$ and the quadratic functions behave rather similarly. However, as we increase $|w|$, the two quadratic functions become rather distinguishable, whereas $F_e(w)$ become more and more \emph{in}distinguishable for the two choices of $e$. Thus, distinguishing whether $e=0.1$ or $e=-0.1$, based only on function values is of $F_e(w)$, is much harder than the quadratic case}
\label{fig:functions}
\end{figure}

\section{Discussion}\label{sec:discussion}

In this paper, we considered the dual settings of bandit and derivative-free stochastic convex optimization. We provided a sharp characterization of the attainable performance for strongly-convex and smooth functions. The results also provide useful lower-bounds for more general settings. We also considered the case of quadratic functions, showing that a ``fast'' $\Ocal(1/T)$ rate is possible in a stochastic setting, even without knowledge of derivatives. Our results have several qualitative differences compared to previously known results which focus on linear functions, such as quadratic dependence on the dimension even for extremely ``nice'' functions, and a provable gap between the attainable performance in bandit optimization and derivative-free optimization.

Our work leaves open several questions. For example, we have only dealt with bounds which hold in expectation, and our lower bounds focused on the dependence on $d,T$, where other problem parameters, such as the Lipschitz constant and strong convexity parameter, are fixed constants. While this follows the setting of previous works, it does not cover situations where these parameters scale with $d$. Finally, while this paper settles the case of strongly-convex and smooth functions, we still don't know what is the attainable performance for general convex functions, as well as the more specific case of strongly-convex (possibly non-smooth) functions. Our $\Omega\left(\sqrt{d^2/T}\right)$ lower bound still holds, but the existing upper bounds are much larger: $\min\left\{\sqrt[4]{d^2/T},\sqrt{d^{32}/T}\right\}$ for convex functions, and $\min\left\{\sqrt[3]{d^2/T},\sqrt{d^{32}/T}\right\}$ for strongly-convex functions (see table \ref{table:results}). We don't know if the lower bound or the existing upper bounds are tight. However, it is the current upper bounds which seem less ``natural'', and we suspect that they are the ones that can be considerably improved, using new algorithms which remain undiscovered.

\acks{We thank John Duchi, Satyen Kale, Robi Krauthgamer and the anonymous reviewers for helpful discussions and comments.}

\bibliography{mybib}

\appendix

\section{Improved Results for Quadratic Functions}\label{sec:ridge}

In \secref{sec:quadratic}, we showed a tight $\Theta(d^2/T)$ bound on the achievable error for quadratic functions, in the derivative-free SCO setting. This was shown under the assumption that the noise $\xi_{\bw}$ is zero-mean and has a second moment bounded by $\max\{1,\norm{\bw}^2\}$. In this appendix, we show how under additional natural assumptions on the noise, one can improve on this result with an efficient algorithm. The main message here is not so much the algorithmic result, but rather to show that the generic noise assumption is important for our lower bounds, and that better algorithms may still be possible for more specific settings.

To give a concrete example, consider the classic setting of ridge regression, where we have labeled training examples $(\bx,y)$ sampled i.i.d. from some distribution over $\reals^d\times \reals$, and our goal is to find some $\bw\in \reals^d$ minimizing
\[
F(\bw) = \frac{\lambda}{2}\norm{\bw}^2+\E_{(\bx,y)}\left[\left(\bw^\top \bx - y\right)^2\right].
\]
In a bandit / derivative-free SCO setting, we can think of each query as giving as the value of
\begin{equation}\label{eq:ridge}
\hat{F}(\bw) = \frac{\lambda}{2}\norm{\bw}^2+\left(\bw^\top \bx - y\right)^2.
\end{equation}
for some specific example $(\bx,y)$, and note that its expected value (over the random draw of $(\bx,y)$) equals $F(\bw)$. Thus, it falls within the setting considered in this paper. However, the noise process is not generic, but has a particular structure. We will show here that one can actually attain an error rate as good as $\Ocal(d/T)$ for this problem.

To formally present our result, it would be useful to consider a more general setting, the ridge regression setting above being a special case. Suppose we can write $F(\bw)$ as $\E[\hat{F}(\bw)]$, where $\hat{F}(\bw)$ decomposes into a deterministic term $R(\bw)$ and a stochastic quadratic term $\hat{G}(\bw)$:
\[
\hat{F}(\bw) ~=~ R(\bw)+\hat{G}(\bw) ~=~ R(\bw)+\left(\bw^\top \hat{A} \bw +\hat{\bb}^\top \bw+\hat{c}\right),
\]
where $\hat{A},\hat{\bb},\hat{c}$ are random variables. We assume that whenever we query a point $\bw$, we get $\hat{F}(\bw)$ for some random realization of $\hat{A},\hat{\bb},\hat{c}$. In general, $R(\bw)$ can be a strongly-convex regularization term, such as $\frac{\lambda}{2}\norm{\bw}^2$ in \eqref{eq:ridge}.

The algorithm we consider, Algorithm \ref{alg:quadratic2}, is a slight variant of Algorithm \ref{alg:quadratic}, which takes this decomposition of $F(\bw)$ into account when constructing its unbiased gradient estimate. Compared to Algorithm \ref{alg:quadratic}, this algorithm also queries at random points further away from $\bw_t$, up to a distance of $\sqrt{d}$. We will assume here that we can always query at such points\footnote{Similar to Algorithm \ref{alg:quadratic}, if one can only query at some distance $\epsilon \sqrt{d}$, where $\epsilon \in (0,1]$, then one can modify the algorithm to handle such cases, with the resulting error bound depending on $\epsilon$.}. We also let $\bar{\Wcal}=\Wcal \cap \{\bw:\norm{\bw}\leq B\}$ in the algorithm, where we recall that $B$ is some known upper bound on $\norm{\bw^*}$.

\begin{algorithm}
\caption{Derivative-Free SCO Algorithm for Decomposable-Quadratic Functions}
\label{alg:quadratic2}
\begin{algorithmic}
\STATE Input: Deterministic term $R(\cdot)$; Strong convexity parameter $\lambda>0$
\STATE Initialize $\bw_1=\mathbf{0}$.
\FOR{$t=1,\ldots,T-1$}
    \STATE Pick $\br\in \{-1,+1\}^d$ uniformly at random
    \STATE Query noisy function value $v$ at point $\bw_t+\br$
    \STATE Let $\tilde{\bg}=\left(v-R\left(\bw_t+\br\right)\right)\br+\bg_R(\bw_t)$, where $\bg_R(\bw)$ is a subgradient of $R(\cdot)$ at $\bw$
    \STATE Let $\bw_{t+1} = \Pi_{\bar{\Wcal}}\left(\bw_{t}-\frac{1}{\lambda t}\tilde{\bg}\right)$
\ENDFOR
\STATE Return $\bar{\bw}_T=\bw_{T}\frac{2}{T}\sum_{t=T/2}^{T}\bw_{t}$.
\end{algorithmic}
\end{algorithm}

We now show that with this algorithm, one can improve on our $\Ocal(d^2/T)$ error upper bound from (\thmref{thm:quadup}).
\begin{theorem}\label{thm:quadup2}
In the setting described above, suppose $\norm{\hat{A}}_2,\norm{\hat{\bb}},|\hat{c}|$ are all at most $1$ with probability $1$, the optimum $\bw^*$ has a norm of at most $B$, and $\norm{\bg_R(\bw)}\leq N$ for any $\bw\in\bar{\Wcal}$. Then under Assumption \ref{assump}, the point $\bar{\bw}_T$ returned by Algorithm \ref{alg:quadratic2} satisfies
\[
\E\left[F(\bar{\bw}_T)- F(\bw^*)\right] ~\leq~ 4(4+5\log(2))\frac{N^2+3d\left(\left(B+1\right)^4+\E\left[\norm{\hat{A}}_F^2\right]\right)}
{\lambda T},
\]
where $\norm{\cdot}_F$ is the Frobenius norm.
\end{theorem}
Note that if we only assume $\norm{\hat{A}}_2\leq 1$, then $\norm{\hat{A}}_F^2$ can be as high as $d$, which leads to an $\Ocal(d^2/T)$ bound, same as in \thmref{thm:quadup}. However, it may be much smaller than that. In particular, for the ridge regression case we considered earlier, $\hat{A}$ corresponds to $\bx \bx^\top$ where $\bx$ is a randomly drawn instance. Under the common assumption that $\norm{\bx}\leq \Ocal(1)$ (independent of the dimension), it follows that $\norm{\bx \bx^\top}_F^2 = \norm{\bx}^4 = \Ocal(1)$. Therefore, $\norm{\hat{A}}_F^2$ is independent of the dimension, leading to an $\Ocal(d/T)$ error upper bound in terms of $d,T$.

We remark that even in this specific setting, the $\Ocal(d/T)$ bound does not carry over to the bandit SCO setting (i.e. in terms of regret), since the algorithm requires us to query far away from $\bw_t$. Also, we again emphasize that this result does not contradict our lower bound in the quadratic case (\thmref{thm:quadlow}), since the setting there included a generic noise term, while here the stochastic ``noise'' has a very specific structure.

As to the proof of \thmref{thm:quadup2}, it is very similar to that of \thmref{thm:quadup}, the key difference being a better moment upper bound on the gradient estimate $\bar{\bg}^2$, as formalized in the following lemma. Plugging this improved bound into the calculations results in the theorem.
\begin{lemma}\label{lem:momentbounds2}
For any $\bw_t$, we have that $\E_{\br,v}[\tilde{\bg}]$ is a subgradient of $F(\bw_t)$, and
\[
\E_{\br,v}[\norm{\tilde{\bg}}^2] \leq 4\left(N^2+3d\left(\left(B+1\right)^4+~\E\left[\norm{\hat{A}}_F^2\right]\right)
\right).
\]
\end{lemma}
\begin{proof}
By definition of $F(\bw_t)$, we note that
\[
\tilde{\bg} = \left(\left(\bw_t+\br\right)^\top \hat{A} \left(\bw_t+\br\right)+\hat{\bb}^\top \left(\bw_t+\br\right)+\hat{c}\right)
\br+\bg_R(\bw_t).
\]
Using a similar calculation to the one in the proof of \lemref{lem:momentbounds}, we have that the expected value of this expression over $\br$ and $\hat{A},\hat{\bb},\hat{c}$ is
\[
2\bw_t^\top \E[\hat{A}]+\E[\hat{\bb}^\top]+\bg_R(\bw_t),
\]
which is a subgradient of $F(\bw_t)$. As to the moment bound, we have
\begin{align}
\E[\norm{\tilde{\bg}}^2] ~&\leq~
\E\left[4\left(\left(\bw_t+\br\right)^\top \hat{A} \left(\bw_t+\br\right)\right)^2\norm{\br}^2+4\left(\hat{\bb}^\top(\bw_t+\br)\right)^2\norm{\br}^2
+4\hat{c}^2\norm{\br}^2+4\norm{\bg_R(\bw_t)}^2\right]\notag\\
&\leq~ 4d~\E\left[\left(\norm{\hat{A}}_2\norm{\bw_t}^2+2\bw_t^\top \hat{A}\br+\br^\top \hat{A} \br\right)^2+2\left(\left(\hat{\bb}^\top \bw_t\right)^2+\left(\hat{\bb}^\top\br\right)^2\right)+1\right]+4N^2\notag\\
&=~ 4d~\E\left[\left(B^2+2\bw_t^\top \hat{A}\br+\br^\top A \br\right)^2+2\left(B^2+\left(\hat{\bb}^\top \br\right)^2\right)+1\right]+4N^2\notag\\
&=~ 12d\left(B^4+4\E\left[\left(\bw_t^\top \hat{A}\br\right)^2\right]+\E\left[\left(\br^\top A\br\right)^2 \right]\right)+8d\left(B^2+\E\left[\left(\hat{\bb}^\top\br\right)^2\right]
\right)+4d+4N^2.\label{eq:gbound}
\end{align}
Letting $\hat{a}_{i,j}$ denote entry $(i,j)$ in $\hat{A}$, and recalling that by definition of $\br$, $\E[r_i r_j] = \mathbf{1}_{i=j}$, we have that
\begin{align*}
\E\left[\left(\br^\top A\br\right)^2\right] ~&=~
\E\left[\left(\sum_{i,j} r_i r_j \hat{a}_{i,j}\right)^2\right]
~=~ \E\left[\sum_{i,j,i',j'} r_i r_j r_{i'} r_{j'} \hat{a}_{i,j} \hat{a}_{i',j'}\right]\\
&=~ \E\left[\sum_{i,j}r_i^2 r_j^2 \hat{a}_{i,j}^2\right]
~=~ \E\left[\sum_{i,j}\hat{a}_{i,j}^2\right] = \E\left[\norm{\hat{A}}_{F}^2\right].
\end{align*}
Also, using the fact that $\E[\br\br^\top]$ is the identity matrix, we have
\begin{align*}
\E\left[\left(\bw_t^\top \hat{A}\br\right)^2\right] ~=~
\E\left[\bw_t^\top \hat{A}\br \br^\top \hat{A}^\top \bw_t\right]
~=~
\E\left[\bw_t^\top \hat{A}\hat{A}^\top \bw_t\right]
~\leq~ \E\left[\norm{\bw_t}^2\norm{\hat{A}}_2^2\right]
~\leq~ B^2.
\end{align*}
Finally, we have
\[
\E\left[\left(\hat{\bb}^\top\br\right)^2\right]
~=~
\E\left[\hat{\bb}^\top \br \br^\top \hat{\bb}\right]
~=~
\E\left[\norm{\hat{\bb}}^2\right] ~\leq~ 1.
\]
Plugging these inequalities back into \eqref{eq:gbound}, we get that
\begin{align*}
\E[\norm{\bg}^2] ~&\leq~ 12d\left(B^4+4B^2+\E[\norm{\hat{A}}_F^2]\right)
+8d\left(B^2+1\right)+4d+4N^2\\
&=~ 4d\left(3B^4+14 B^2 +3+3~\E\left[\norm{\hat{A}}_F^2\right]\right)+4N^2\\
&\leq~ 12d\left((B+1)^4+\E\left[\norm{\hat{A}}_F^2\right]\right)+4N^2,
\end{align*}
from which the lemma follows.
\end{proof}

\section{Additional Proofs}\label{sec:technical}

\subsection{Proof of \lemref{lem:momentbounds}}

By the way $\br$ is picked, we have that $\E_{\br}[\br_i \br_j]=\mathbf{1}_{i=j}$ and that $\E_{\br}[\br_i \br_j \br_k]=0$ for all $i,j,k$. Thus, letting $\E$ denote expectation w.r.t. $\br$ and the random function values, we have
\begin{align*}
\E[\tilde{\bg}] ~&=~
\E\left[\frac{\sqrt{d}v}{\epsilon} \br\right]\\
&=~
\E\left[\frac{\sqrt{d}}{\epsilon}\left(\left(\bw+\frac{\epsilon}{\sqrt{d}}\br\right)^\top A \left(\bw+\frac{\epsilon}{\sqrt{d}}\br\right)+\bb^\top \left(\bw+\frac{\epsilon}{\sqrt{d}}\br\right)+c+\xi_{\bw_t+\frac{\epsilon}{\sqrt{d}}\br}\right)
\br \right]\\
&=~
\E\left[\frac{\sqrt{d}}{\epsilon}\left(\bw^\top A \bw+ \bb^\top \bw + c + \xi_{\bw_t+\frac{\epsilon}{\sqrt{d}}\br}\right)\br+\frac{\epsilon}{\sqrt{d}}\left(\br^\top A \br\right)\br\right]
+\E\left[\left(2\bw^\top A \br\right)\br+\left(\bb^\top \br\right)\br\right]\\
&=~
0+2\bw^\top A + \bb^\top + 0
~=~
\nabla F(\bw).
\end{align*}
Also, by the assumptions on $A,\bb,c$ and the assumptions on the noise $\xi_{\bw}$, we have
\begin{align*}
\E[\norm{\tilde{\bg}}^2] ~&=~ \E\left[\frac{d v^2}{\epsilon^2}\norm{\br}^2\right]
~=~ \frac{d^2}{\epsilon^2} \E[v^2]
~=~ \frac{d^2}{\epsilon^2} \E\left[\left(F \left(\bw_t+\frac{\epsilon}{\sqrt{d}}\br\right)+\xi_{\bw_t+\frac{\epsilon}{\sqrt{d}}\br}\right)^2\right]\\
&\leq \frac{d^2}{\epsilon^2} \E\left[2\left(F \left(\bw_t+\frac{\epsilon}{\sqrt{d}}\br\right)\right)^2+2\xi_{\bw_t+\frac{\epsilon}{\sqrt{d}}\br}^2\right]\\
&\leq~ \frac{2d^2}{\epsilon^2}\left(\sup_{\bw:\norm{\bw}\leq B+\epsilon} (F(\bw))^2+\max\left\{1,\norm{\bw_t+\frac{\epsilon}{\sqrt{d}}\br}^2\right\}\right)\\
&\leq~ \frac{2d^2}{\epsilon^2}\left(\sup_{\bw:\norm{\bw}\leq B+\epsilon} \left(\bw^\top A \bw+\bb^\top \bw+c\right)^2+(B+1)^2\right)\\
&\leq~ \frac{2d^2}{\epsilon^2} \left(\left((B+\epsilon)^2\text+(B+\epsilon)+1\right)^2+(B+1)^2\right)\\
&\leq~ \frac{4d^2}{\epsilon^2}(B+1)^4
\end{align*}
as required.

\subsection{Proof of \lemref{lem:probdkl}}

We have the following:
\begin{align}
&\E\left[\sum_{i=1}^{d}\mathbf{1}_{\bar{w_i}e_i < 0}\right]
~=~ \sum_{i=1}^{d}\Pr\left(\bar{w_i}e_i < 0\right)\notag\\
&=~ \frac{1}{2}\sum_{i=1}^{d}\left(\Pr(\bar{w}_{i}<0 | e_i>0)+\Pr(\bar{w}_{i}>0 | e_i<0)\right)\notag\\
&=~
\frac{1}{2}\left(d-\sum_{i=1}^{d}\left(\Pr(\bar{w}_{i}>0 | e_i>0)-\Pr(\bar{w}_{i}>0 | e_i<0)\right)\right)\notag\\
&\geq~
\frac{d}{2}\left(1-\frac{1}{d}\sum_{i=1}^{d}\left|\Pr(\bar{w}_{i}>0 | e_i>0)-\Pr(\bar{w}_{i}>0 | e_i<0)\right|\right)\notag\\
&\geq~
\frac{d}{2}\left(1-\sqrt{\frac{1}{d}\sum_{i=1}^{d}\left(\Pr(\bar{w}_{i}>0 | e_i>0)-\Pr(\bar{w}_{i}>0 | e_i<0)\right)^2}\right),
\label{eq:prsum}
\end{align}
where the last inequality is by the fact that for any values $a_1,\ldots,a_d$, it holds that $|a_1|+\ldots+|a_d| \leq \sqrt{d} \sqrt{a_1^2+\ldots+a_d^2}$.

Consider (without loss of generality) the term corresponding to the first coordinate, namely
\[
\left(\Pr(\bar{w}_{1}>0 | e_1>0)-\Pr(\bar{w}_{1}>0 | e_1<0)\right)^2.
\]
This term equals
\begin{align*}
&\left(\sum_{e_2,\ldots,e_d}\Pr(\{e_j\}_{j=2}^d)\left(\Pr\left(\bar{w}_{1}>0| e_1>0,\{e_j\}_{j=2}^d\right)-\Pr\left(\bar{w}_{1}>0|e_1<0,\{e_j\}_{j=2}^d\right)\right)\right)^2
\\
&\leq~ \sum_{e_2,\ldots,e_d}\Pr(\{e_j\}_{j=2}^d)\left(\Pr\left(\bar{w}_{1}>0| e_1>0,\{e_j\}_{j=2}^d\right)-\Pr\left(\bar{w}_{1}>0|e_1<0,\{e_j\}_{j=2}^d\right)\right)^2\\
&\leq~
\sup_{e_2,\ldots,e_d}\left(\Pr\left(\bar{w}_{1}>0| e_1>0,\{e_j\}_{j=2}^d\right)-\Pr\left(\bar{w}_{1}>0|e_1<0,\{e_j\}_{j=2}^d\right)\right)^2
\end{align*}
By Pinsker's inequality and the assumption that $\bar{\bw}_T$ is a deterministic function of $v_1,\ldots,v_T$, this expression is at most
\[
\frac{1}{2}D_{kl}\left(\Pr\left(v_1,\ldots,v_T|e_1>0,\{e_j\}_{j=2}^d\right)||\Pr\left(v_1,\ldots,v_T|e_1<0,\{e_j\}_{j=2}^d\right)\right),
\]
where $D_{kl}(P||Q)$ is the Kullback-Leibler divergence between the two distributions. By the chain rule (see e.g. \cite{covthom06}), we can upper bound the above by
\[
\frac{1}{2}\sum_{t=1}^{T} D_{kl}\left(\Pr\left(v_t|e_1>0,\{\be_j\}_{j=2}^{d},\{v_l\}_{l=1}^{t-1}\right)
~||~
\Pr\left(v_t|e_1<0,\{\be_j\}_{j=2}^{d},\{v_l\}_{l=1}^{t-1}\right)\right).
\]
Plugging these bounds back into \eqref{eq:prsum}, the result follows.

\subsection{Proof of \thmref{thm:quadlowregret}}

We may assume without loss of generality that $T\geq d^2$, and it is enough to show that the expected average regret is at least $0.02\sqrt{d^2/T}$. This is because if there was a strategy with $<0.02$ average regret after $T<d^2$ rounds, then for the case of $d^2$ rounds, we could just run that strategy for $T$ rounds, compute the average $\bar{\bw}_T$ of all points played so far, and then repeatedly choose $\bar{\bw}_T$ in the remaining rounds. By Jensen's inequality, this would imply a $<0.02$ average regret after $d^2$ rounds, in contradiction.

Let $\bar{\bw}_T$ be an arbitrary deterministic function of $\bw_1,\ldots,\bw_T$. A proof identical to that of \thmref{thm:quadlow}, up to \eqref{eq:quadloweq}, implies that for any $\mu>0$, there exists a quadratic function of the form
\[
F_{\be}=\frac{1}{2}\norm{\bw}^2-\inner{\be,\bw},
\]
with $\be\in \{-\mu,\mu\}^d$, such that
\[
\E[F_{\be}(\bar{\bw}_T)-F_{\be}(\bw^*)] \geq
\E\left[\frac{d\mu^2}{4}\left(1-\sqrt{\frac{\mu^2}{d}\sum_{t=1}^{T}
\min\left\{\norm{\bw_t}^2,\frac{1}{\norm{\bw_t}^2}\right\}}\right)\right].
\]
In particular, letting $\bar{\bw}_T=\frac{1}{T}\sum_{t=1}^{T}\bw_t$, using Jensen's inequality, and discarding the $\min$, we get that
\begin{equation}\label{eq:regretlow1}
\E\left[\frac{1}{T}\sum_{t=1}^{T}F_{\be}(\bw_t)-F_{\be}(\bw^*)\right]
~\geq~\frac{d\mu^2}{4}\left(1-\sqrt{\frac{\mu^2}{d}\sum_{t=1}^{T}
\norm{\bw_t}^2}\right).
\end{equation}
However, we also know that by strong convexity of $F_{\be}$, we have
\begin{equation}\label{eq:regretlow2}
\E\left[\frac{1}{T}\sum_{t=1}^{T}F_{\be}(\bw_t)-F_{\be}(\bw^*)\right]\geq \frac{1}{2T}\sum_{t=1}^{T}\norm{\bw_t-\be}^2.
\end{equation}
Using the fact that
\[
\norm{\bw_t}^2 = \norm{\bw_t-\be+\be}^2\leq (\norm{\bw_t-\be}+\norm{\be})^2 \leq 2\norm{\bw_t-\be}^2+2\norm{\be}^2,
\]
we get that
\[
\norm{\bw_t-\be}^2 ~\geq~ \frac{1}{2}\norm{\bw_t}^2-\norm{\be}^2
~=~ \frac{1}{2}\norm{\bw_t}^2-d\mu^2.
\]
Substituting into \eqref{eq:regretlow2} and slightly manipulating the resulting inequality, we get
\[
\sum_{t=1}^{T}\norm{\bw_t}^2 \leq 4T \E\left[\frac{1}{T}\sum_{t=1}^{T}F_{\be}(\bw_t)-F_{\be}(\bw^*)\right]+2Td\mu^2.
\]
For simplicity, denote the average regret term $\E\left[\frac{1}{T}\sum_{t=1}^{T}F_{\be}(\bw_t)-F_{\be}(\bw^*)\right]$ by $R$. Substituting the expression above into \eqref{eq:regretlow1}, we get
\[
R ~\geq~ \frac{d\mu^2}{4}\left(1-\sqrt{\frac{\mu^2}{d}\left(4TR+2Td\mu^2\right)}\right)
~\geq~ \frac{d\mu^2}{4}\left(1-\sqrt{\frac{4\mu^2 TR}{d}}-\sqrt{2T\mu^4}\right).
\]
Rearranging and simplifying, we get
\[
R+\frac{\sqrt{dT}}{2}\mu^3\sqrt{R}+\frac{d\mu^2}{4}\left(\mu^2\sqrt{2T}-1\right)\geq 0.
\]
The equation above can be seen as a quadratic function of $\sqrt{R}$, with the roots
\[
\frac{1}{2}\left(-\frac{\sqrt{dT}}{2}\mu^3\pm \sqrt{\left(\frac{\sqrt{dT}}{2}\mu^3\right)^2+d\mu^2\left(1-\mu^2\sqrt{2T}\right)}\right).
\]
Now, recall that $\mu$ is a free parameter that we can choose at will. If we choose it so that $1-\mu^2\sqrt{2T}> 0$, then it is easy to show that we get two roots, one strictly positive and one strictly negative. Since we know $\sqrt{R}$ is a nonnegative quantity, we get that
\begin{align*}
\sqrt{R} ~&\geq~ \frac{1}{2}\left(-\frac{\sqrt{dT}}{2}\mu^3+ \sqrt{\left(\frac{\sqrt{dT}}{2}\mu^3\right)^2+d\mu^2\left(1-\mu^2\sqrt{2T}\right)}\right)\\
&=~ \frac{\sqrt{d}\mu}{2}\left(-\frac{\sqrt{T}}{2}\mu^2+\sqrt{\frac{T}{4}\mu^4+1-\mu^2\sqrt{2T}}\right).
\end{align*}
Finally, choosing $\mu = T^{-1/4}/2$ (which indeed satisfies $1-\mu^2\sqrt{2T}> 0$), and simplifying, we get
\[
\sqrt{R} \geq 0.17\sqrt{\frac{d}{\sqrt{T}}}.
\]
Recalling that $R$ is the expected average regret, it only remains to take the square of the two sides. We note that since we assume $T\geq d^2$, then $\norm{\be} = \sqrt{d}\mu=\sqrt{\sqrt{d^2/T}}/2\leq 1/2$, as specified in the theorem statement.

\subsection{Proof of \thmref{thm:convex}}

Let $\mu>0$ be a parameter to be determined later. As discussed in the text, we will look at functions of the form
\begin{equation}\label{eq:strf}
F_{\be}(\bw) = \norm{\bw}^2-\sum_{i=1}^{d}\frac{e_i w_i}{1+(w_i/e_i)^2},
\end{equation}
where $\be$ is uniformly distributed on $\left\{-\mu,+\mu\right\}^d$. Our goal will be to prove a lower bound on the expected optimization error over the randomized choice of $F_{\be}$, with respect to deterministic querying strategies. As explained in the proof of \thmref{thm:quadlow}, this would imply the existence of some fixed $F_{\be}$ such that the expected optimization error over a (possibly randomized) querying strategy is the same.

We will need the following properties of $F_{\be}$:
\begin{lemma}\label{lem:function}
For any $\mu>0$ and any $\be\in \{-\mu,+\mu\}^d$, the function $F_{\be}$ in \eqref{eq:strf} is:
\begin{itemize}
\item $0.5$-Strongly convex and $3.5$-smooth
\item $2+\sqrt{2d}\mu$-Lipschitz for any $\bw$ such that $\norm{\bw}\leq 1$.
\item $F_{\be}$ is globally minimized at $\bw^*=c\be$, where $c=0.3489... \geq 1/3$
\item For any $\be'\in \{-\mu,+\mu\}^d$ which differs from $\be$ in a single coordinate, and for any $\bw\in \reals^d$, it holds that $|F_{\be}(\bw)-F_{\be'}(\bw)|\leq \mu^2$.
\end{itemize}
\end{lemma}
\begin{proof}
Note that we can write the function $F_{\be}(\bw)$ as $\sum_{i=1}^{d} g_{e_i}(w_i)$, where
\[
g_a(x) = x^2-\frac{a x}{1+(x/a)^2}.
\]
It is not hard to realize that to prove the lemma, it is enough to prove that:
\begin{enumerate}
\item $g_a(x)$ is $0.5$-strongly convex and $3.5$-smooth;\label{item:strsmooth}
\item $|g'_a(x)|$ is at most\footnote{Since this would imply that $\norm{\nabla F_{\be}(\bw)}$ is at most $\sqrt{\sum_{i=1}^{d}(2|w_i|+\mu)^2} \leq \sqrt{\sum_{i=1}^{d}(4w_i^2+2\mu^2)} \leq \sqrt{\sum_{i=1}^{d}(4w_i^2)}+\sqrt{\sum_{i=1}^{d}(2\mu^2)} =
    2\norm{\bw}+\sqrt{2d}\mu$, which is at most $2+\sqrt{2d}\mu$ for any $\bw$ in the unit ball.}
     $2|x|+|a|$;\label{item:lipschitz}
\item For all $\mu$, $|g_{\mu}(x)-g_{-\mu}(x)| \leq \mu^2$; \label{item:diffbounnd}
\item $g_{a}(x)$ is minimized at $ca$ where $c=0.3489...$.\label{item:min}
\end{enumerate}

To show item \ref{item:strsmooth}, we calculate the second derivative of $g_a(x)$, which is
\[
2\left(1+\frac{a^3 x(3a^2-x^2)}{(a^2+x^2)^3}\right).
\]
By definition of strong convexity and smoothness, it is enough to show that this term is always at least $0.5$ and at most $3.5$.
Substituting $x=ay$ and simplifying, we get
\[
2\left(1+\frac{y(3-y^2)}{(1+y^2)^3}\right).
\]
It is a straightforward exercise to verify that $\left|\frac{y(3-y^2)}{(1+y^2)^3}\right|$ is at most $3/4$ for all $y\in \reals$, hence the expression above is always in $[0.5,3.5]$ as required.

As to item \ref{item:lipschitz}, we note that
\[
g'_a(x) ~=~ 2x-\frac{a^5-a^3 x^2}{(a^2+x^2)^2}
~=~ 2x-a\frac{1-(x/a)^2}{(1+(x/a)^2)^2}.
\]
For any value of $x/a$, the value of the fraction above is easily verified to be at most $1$, hence we can upper bound $|g'_a(x)|$ by $2|x|+|a|$ as required.

As to item \ref{item:diffbounnd}, we have
\[
|g_{\mu}(x)-g_{-\mu}(x)| = \frac{2|\mu x|}{1+(x/\mu)^2} = \mu^2 \frac{2|\mu x|}{\mu^2+x^2} \leq \mu^2,
\]
where the last step uses $\mu^2+x^2 \geq 2|\mu x|$, which follows from the identity $(\mu+|x|)^2\geq 0$.

Finally, as to item \ref{item:min}, we note that this function can be equivalently written as
\[
g_a(x) = a^2\left((x/a)^2-\frac{(x/a)}{1+(x/a)^2}\right).
\]
Substituting $x=ay$, we get $a^2 (y^2-y/(1+y^2))$. A numerical calculation reveals that the minimizing value of $y$ is $0.3489...$, hence the minimizing value of $x$ is $0.3489...*a$ as required.
\end{proof}

We now begin to derive the lower bound. Using strong convexity and the lemma, we have
\begin{align}
\E[F(\bar{\bw}_T)-F(\bw^*)] ~&\geq~ \E\left[\frac{1}{4}\norm{\bar{\bw}_T-\bw^*}^2\right]
~=~\frac{1}{4} \E\left[\sum_{i=1}^{d}(\bar{w}_i-w^*_i)^2\right]
~\geq~ \frac{1}{4} \E\left[\sum_{i=1}^{d}(w^*_i)^2 \mathbf{1}_{\bar{w}_i w^*_i <0}\right]\notag\\
&\geq \frac{1}{4} \E\left[\sum_{i=1}^{d}\left(\frac{e_i}{3}\right)^2\mathbf{1}_{\bar{w}_i e_i <0}\right]
~=~ \frac{\mu^2}{36}\E\left[\sum_{i=1}^{d}\mathbf{1}_{\bar{w}_i e_i <0}\right]
\label{eq:lowstr}
\end{align}

We now lower bound this term using \lemref{lem:probdkl}. To do so, we need to upper bound the KL divergence of the query values at round $t$ under the two hypotheses $e_i=+\mu$ and $e_i=-\mu$, the other coordinates being fixed. We assume each noise term $\xi_{\bw}$ is a standard Gaussian random variable. Thus, the query value that we see is distributed as
\[
F_{\be}(\bw_t)+\xi_{\bw} ~=~ \norm{\bw}^2-\sum_{j=1}^{d}\frac{e_j w_j}{1+(w_j/e_j)^2}+\xi_{\bw}.
\]
where one of the coordinates $i$ of $\be$ is either $+\mu$ or $-\mu$ and the other coordinates are fixed. This is a Gaussian distribution, with mean $F_{\be}(\bw_t)$ and variance $1$. By \lemref{lem:function}, the difference between the two means under the two cases $e_i=+\mu$, $e_i=-\mu$ is at most $\mu^2$, so by \lemref{lem:gaussians}, the KL-divergence is at most $\mu^4/2$.
Using \lemref{lem:probdkl}, this implies that \eqref{eq:lowstr} is at least
\[
\frac{d\mu^2}{72}\left(1-\sqrt{\frac{1}{d}\sum_{i=1}^{d}\sum_{t=1}^{T}\frac{\mu^4}{2}}\right)
~=~
\frac{d\mu^2}{72}\left(1-\sqrt{\frac{T\mu^4}{2}}\right).
\]
Picking $\mu=T^{-1/4}$, we get a lower bound of $d/144\sqrt{2T} > 0.004 \sqrt{d^2/T}$.

Finally, note that for this choice of $\mu$, by \lemref{lem:function}, our function $F_{\be}$ (for any realization of $\be$) is $2+\sqrt{2d/\sqrt{T}}$- Lipschitz in the unit ball, and has a global minimum with norm at most $0.35\sqrt{d/\sqrt{T}}$. If $T\geq d^2$, the Lipschitz parameter is at most $4$ and the global minimum is inside the unit ball, satisfying the requirements in the theorem statement. If $T<d^2$, then the bound cannot be better than what we would obtain for $T=d^2$ (the argument is similar to the one in the proof of \thmref{thm:quadlowregret}), which is $0.004$. Thus, for any $T$, the bound is at least
\[
\min\left\{0.004,0.004\sqrt{\frac{d^2}{T}}\right\}
~=~
0.004 \min\left\{1,\sqrt{\frac{d^2}{T}}\right\}
\]
as required.

\end{document}